\documentclass[]{article}
\usepackage{ifpdf}
\usepackage{proceed2e-uai}
\usepackage{graphicx}
\usepackage{amssymb,amsfonts,amsbsy,amsmath,amsthm}

\ifpdf
\usepackage[%
  pdftitle={Instructions for use of the document class
    elsart},%
  pdfauthor={Simon Pepping},%
  pdfsubject={The preprint document class elsart},%
  pdfkeywords={instructions for use, elsart, document class},%
  pdfstartview=FitH,%
  bookmarks=true,%
  bookmarksopen=true,%
  breaklinks=true,%
  colorlinks=true,%
  linkcolor=blue,anchorcolor=blue,%
  citecolor=blue,filecolor=blue,%
  menucolor=blue,pagecolor=blue,%
  urlcolor=blue]{hyperref}
\else
\usepackage[%
  breaklinks=true,%
  colorlinks=true,%
  linkcolor=blue,anchorcolor=blue,%
  citecolor=blue,filecolor=blue,%
  menucolor=blue,pagecolor=blue,%
  urlcolor=blue]{hyperref}
\fi

\theoremstyle{plain}
\newtheorem{thm}{Theorem}[section]
\newtheorem{theorem}[thm]{Theorem}

\newtheorem{corollary}[thm]{Corollary}

\newtheorem{lemma}[thm]{Lemma}

\newtheorem{proposition}[thm]{Proposition}

\theoremstyle{definition}

\newtheorem{remark}[thm]{Remark}

\newtheorem{example}[thm]{Example}

\newtheorem{definition}[thm]{Definition}

\newcommand{\set}[1]{ \{ #1 \} }

\newcommand{\DeltaF}[2]{\Delta{#1}(#2)}
\newcommand{\Lat}[3]{\mathcal{L}(#1,#2|#3)}

\newcommand{\indep}[3]{ \ensuremath{I(#1, #2 | #3)} }

\newcommand{\amodels}{\models^{\it a}}  
\newcommand{\mmodels}{\models^{\it m}}  
\newcommand{\nmmodels}{\not\models^{\it m}}  

\newcommand\T{\rule{0pt}{2.2ex}}
\newcommand\B{\rule[-1.0ex]{0pt}{0pt}}

\begin{document}

\title{On the Conditional Independence Implication Problem: \\A Lattice-Theoretic Approach\thanks{A version of this paper appeared in the Proceedings of the 24th Conference on Uncertainty in AI, 2008}}

\author{{\bf Mathias Niepert} \\  
Department of Computer Science\\  
Indiana University\\
Bloomington, IN, USA \\  
\texttt{mniepert@cs.indiana.edu}
\And
{\bf Dirk Van Gucht} \\  
Department of Computer Science\\  
Indiana University \\
Bloomington, IN, USA \\ 
\texttt{vgucht@cs.indiana.edu}
\And 
{\bf Marc Gyssens} \\
Department WNI \\  
Hasselt University \& Transnational \\
University of Limburg, Belgium \\ 
\texttt{marc.gyssens@uhasselt.be}
}
\maketitle

\begin{abstract}
\vspace{-3mm}
A lattice-theoretic framework is introduced that permits the study of
the conditional independence (CI) implication problem relative to the class of discrete probability measures.  Semi-lattices are associated with CI statements and a finite, sound and complete  inference system relative to semi-lattice inclusions is presented.  This system
is shown to be (1) sound and complete for saturated CI
statements, (2) complete for general CI statements, and (3) sound and
complete for stable CI statements.  These results yield a criterion that can be used to falsify instances of the implication problem and several heuristics are derived that approximate this ``lattice-exclusion" criterion in polynomial time. Finally, we provide experimental results that relate our work to results obtained from other existing inference algorithms.
\end{abstract}
\vspace{-3mm}

\section{Introduction}
\label{sec-introduction}
\vspace{-2mm}
Conditional independence is an important concept in many
calculi for dealing with knowledge and uncertainty in artificial
intelligence. The notion plays a fundamental role for learning and
reasoning in probabilistic systems which are successfully employed in
areas such as computer vision, computational biology, and
robotics. Hence, new theoretical findings and algorithmic improvements
have the potential to impact many fields of research.
A central issue for reasoning about conditional independence is the
\emph{probabilistic conditional independence implication problem}, that is, to
decide whether a CI statement is entailed by a set of other CI
statements relative to the class of discrete probability measures.  While it
remains open whether this problem is decidable, it is known that there
exists no finite, sound and complete inference system (Studen{\'y}~\cite{STU:1992}).  However, there exist finite sound inference systems
that have attracted special interest.  The most prominent is the \emph{semi-graphoid} axiom system which was introduced as a set of sound inference rules relative to the class of discrete probability measures (Pearl~\cite{GEI:1988}). One of the main contributions of this paper is to extend the semi-graphoids to a \emph{finite} inference system, denoted by $\mathcal{A}$, which we will show to
be (1) sound and complete for saturated CI statements, (2) complete
for general CI statements, and (3) sound and complete for stable
CI statements (de Waal and van der Gaag~\cite{WAAL:2004}), all relative to the class of discrete probability measures.  

\looseness=-1 The techniques we use to obtain these results are made possible through the introduction of a \emph{lattice-theoretic framework}.  In this
approach, semi-lattices are associated with conditional independence statements, and $\mathcal{A}$ is shown to be sound and complete relative to
certain inclusion relationships on these semi-lattices.  To make the
connection between this framework and the conditional independence
implication problem, we first link the latter to an addition-based version of
the problem. In particular, we introduce the \emph{additive
implication problem} for CI statements relative to certain classes of
real-valued functions and specify properties of these classes that
guarantee soundness and completeness, respectively, of $\mathcal{A}$ for the
implication problem.  Through the concept of
\emph{multi-information functions} induced by probability
measures (Studen{\'y}~\cite{STU:2005}), we link the additive
implication problem for this class of functions to the probabilistic CI implication problem.
The combination of the lattice-inclusion techniques and the
completeness result for conditional independence statements allows us to derive criteria that can be used to falsify instances of the implication problem.  We show experimentally that these criteria, some of which can be tested for in polynomial time, work very effectively, and we relate the
experimental results to those obtained from a \emph{racing algorithm}
introduced by Bouckaert and Studen{\'y}~\cite{BOU:2007}.

\section{CI Statements and System $\mathcal{A}$}
\label{SystemA}

We define CI statements and introduce the finite inference system
$\mathcal{A}$ for reasoning about the conditional independence
implication problem.  We will often write $AB$ for the union $A \cup
B$, $ab$ for the set $\set{a,b}$, and $a$ for the singleton set $\set{a}$ whenever the interpretation is
clear from the context.  Throughout the paper, $S$ denotes a finite implicit 
set of statistical variables.

\begin{definition}
The expression $\indep{A}{B}{C}$ where $A$, $B$, and $C$ are
pairwise disjoint subsets of $S$ is called a \emph{conditional
independence (CI) statement}.  If $ABC = S$ we say that
$\indep{A}{B}{C}$ is \emph{saturated}.  If either $A=\emptyset$ and/or
$B=\emptyset$ we say that $\indep{A}{B}{C}$ is \emph{trivial}.
\end{definition}

\begin{figure}[h!]
\begin{tabular}{|l|}
\hline \\
\begin{tabular}{lll}
$\indep{A}{\emptyset}{C}$ & {\bf Triviality} \\ 
$\indep{A}{B}{C} \rightarrow \indep{B}{A}{C}$ & {\bf Symmetry} \\ 
$\indep{A}{BD}{C} \rightarrow \indep{A}{D}{C}$ & {\bf Decomposition} \\ 
$\indep{A}{B}{CD} \wedge \indep{A}{D}{C}$ & {\bf Contraction}\\
$\quad\rightarrow \indep{A}{BD}{C}$ & \\ 
\\
$\indep{A}{B}{C} \rightarrow \indep{A}{B}{CD}$ & {\bf Strong union} \\ 
$\indep{A}{B}{C} \wedge \indep{D}{E}{AC} \wedge $ & {\bf Strong } \\
$\quad \indep{D}{E}{BC} \rightarrow \indep{D}{E}{C}$ & \ \ \ \ \  {\bf contraction}\\
\\ 
\end{tabular}
\\ \hline
\end{tabular}
\caption{\label{fig-inferencerules} 
The inference rules of system $\mathcal{A}$. }
\end{figure}

The set of inference rules in Figure~\ref{fig-inferencerules}
will be denoted by $\mathcal{A}$.  The \emph{triviality},
\emph{symmetry}, \emph{decomposition}, and \emph{contraction} rules
are part of the semi-graphoid axioms (Geiger \cite{GEI:1988}). \emph{Strong union} and \emph{strong contraction} are two additional inference rules. Note that
\emph{strong union} is not a sound inference rule relative to the class of discrete probability measures. The derivability of a CI statement $c$ from a set of CI statements
$\mathcal{C}$ under the inference rules of system $\mathcal{A}$ is denoted by
$\mathcal{C}\vdash c$. The \emph{closure} of $\mathcal{C}$ under
$\mathcal{A}$, denoted $\mathcal{C}^+$, is the set $\set{c \mid
\mathcal{C} \vdash c}$.

\begin{lemma}[de Waal and van der Gaag \cite{WAAL:2004}]
The inference rule \emph{composition} 
\begin{eqnarray*}
\label{composition}
\indep{A}{B}{C} \wedge \indep{A}{D}{C} \rightarrow
\indep{A}{BD}{C} \ \ \mbox{\bf Composition}
\end{eqnarray*}
can be derived using \emph{strong union} and \emph{contraction}.
\end{lemma}

\section{Lattice-Theoretic Framework}
\label{sec-framework}
\vspace{-1mm}
First, we introduce the lattice-theoretic framework which is at the core
of the theory developed in this paper.  The approach we take is made possible 
through the association of conditional independence statements with semi-lattices.  In this
section, we prove that inference system $\mathcal{A}$ is sound and complete
relative to specific semi-lattice inclusions.  This result forms the
backbone of our work on the conditional independence
implication problem.

\subsection{Semi-Lattices of CI Statements}
\label{semi-lattices-CI}

Given two subsets $A$ and $B$ of $S$, we will write $[A, B]$ for the lattice
$\{U \mid A\subseteq U\ \&\ U\subseteq B\}$.  We will now associate
semi-lattices with conditional independence statements.

\begin{definition}
\label{def-latdec} 
Let $\indep{A}{B}{C}$ be a CI statement.  The \emph{semi-lattice} of
$\indep{A}{B}{C}$ is defined by $\Lat{A}{B}{C} = [C, S] - ([A, S] \cup
[B, S])$. 
\end{definition}

We will often write $\mathcal{L}(c)$ to denote the
semi-lattice of a conditional independence statement $c$, and $\mathcal{L}(\mathcal{C})$ to denote the union of semi-lattices, $\bigcup_{c' \in
\mathcal{C}}\mathcal{L}(c')$, of a set of conditional independence statements $\mathcal{C}$. Using the notion of \emph{witnesses} of a conditional independence statement, we can rewrite the associated semi-lattice as a difference-free union of lattices.

\begin{definition}
\label{def-witness}
Let $\indep{A}{B}{C}$ be a CI statement. The set of all witness sets
of $\indep{A}{B}{C}$ is defined as $\mathcal{W}(A, B | C) = \set{
\set{a, b} \mid a \in A \mbox{ and } b \in B }$.
\end{definition}

Note that if $\indep{A}{B}{C}$ is trivial, then $\mathcal{W}(A,B|C) =\emptyset$.

\begin{lemma}
\label{lem-witnessdec}
Let $c = \indep{A}{B}{C}$ be a CI statement. Then
$\mathcal{L}(c) = \bigcup_{W \in \mathcal{W}(c)}[C, \overline{W}]$.
\end{lemma}

\begin{example}
Let $S = \set{a, b, c, d}$ and let $\indep{bc}{d}{a}$ be a CI statement. Then, $\Lat{bc}{d}{a} = [a, S] - ([bc, S] \cup [d, S])  =  \set{a, ab, ac}$. Furthermore, $\mathcal{W}(bc, d | a) = \set{bd, cd}$ and, therefore, 
$\Lat{bc}{d}{a} = [a, ac] \cup [a, ab] = \set{a, ab, ac}$, using Lemma~\ref{lem-witnessdec}.
\end{example}

\subsection{Soundness and Completeness of Inference System $\mathcal{A}$
for Semi-Lattice Inclusion}
\label{sound-complete-A}

We will prove that system $\mathcal{A}$ is sound and
complete relative to semi-lattice inclusion.  First, we show that if a
CI statement can be derived from a set of CI statements under 
$\mathcal{A}$, then we have a set inclusion relationship between their
associated semi-lattices.

\begin{proposition}
\label{prop-easy}
Let $\mathcal{C}$ be a set of CI statements, and let $c$ be a CI
statement. If $\mathcal{C}\vdash c$, then $\mathcal{L}(\mathcal{C}) \supseteq \mathcal{L}(c)$.
\end{proposition}

\begin{proof}
We prove the statement for {strong contraction}. The proofs
for the other inference rules in $\mathcal{A}$ are analogous and are omitted. Let $U
\in \Lat{D}{E}{C}$. Then $U \supseteq C$. If $U \supseteq A$, then $U
\in \Lat{D}{E}{AC}$. If $U \supseteq B$, then $U \in
\Lat{E}{D}{BC}$. If $U \nsupseteq A$ and $U \nsupseteq B$, then $U \in
\Lat{A}{B}{C}$.
\end{proof}

A CI statement can be equivalent to a set of other CI statements with respect to the inference system $\mathcal{A}$.  The following definition of a \emph{witness decomposition} of a CI statement is aimed to prove this property.

\begin{definition}
\label{def-decompositions}
The \emph{witness decomposition} of the CI statement $\indep{A}{B}{C}$
is defined by $wdec(A, B | C) := \set{\indep{a}{b}{C} \mid a \in A
\mbox{ and } b \in B}.$
\end{definition}

A useful property of the witness decomposition of a CI statement is that its closure under $\mathcal{A}$ is the same as the closure of the CI statement itself.  In addition, the semi-lattice of a CI statement is equal to the semi-lattice of its witness decomposition.

\begin{proposition}
\label{lem-closure-equiv}
Let $c$ be a CI statement. (1) $\set{ c }^+ = wdec(c)^+$; and
(2) $\mathcal{L}(c) = \bigcup_{ c' \in wdec(c)} \mathcal{L}(c')$.
\end{proposition}

\begin{proof}
To prove the first statement, let $c = \indep{A}{B}{C}$ and
$\indep{a}{b}{C} \in wdec(c)$.  Then $\indep{a}{b}{C}$ can be derived
from $\indep{A}{B}{C}$ by applications of the \emph{decomposition}
rule. Hence, $wdec(c)^+ \subseteq \set{ c }^+$.  By
Definition~\ref{def-decompositions} we know that for every $a \in A$
and for all $b \in B$ one has $\indep{a}{b}{C} \in wdec(c)$.  By repeatedly
applying \emph{composition}, we can infer the CI statement
$\indep{a}{B}{C}$. Hence, for all $a \in A$, one has $\indep{a}{B}{C}
\in wdec(c)^+$ and by \emph{symmetry} $\indep{B}{a}{C} \in
wdec(c)^+$.  Again, by applying \emph{composition} repeatedly, we can infer
$\indep{B}{A}{C}$ and by \emph{symmetry} $\indep{A}{B}{C}$. Hence, $\set{ c
}^+ \subseteq wdec(c)^+$. 

To prove the second statement, let $\indep{a}{b}{C} \in wdec(c)$ and
$W = \set{a, b}$. Then $\mathcal{L}(a, b | C) = [C,
\overline{W}]$. The statement now follows directly from
Definition~\ref{def-decompositions} and Lemma~\ref{lem-witnessdec}.
\end{proof}

We are now in the position to prove the main result concerning the soundness and completeness of the inference system $\mathcal{A}$ for semi-lattice inclusion.

\begin{theorem}
\label{theo-derivablelatdec}
Let $\mathcal{C}$ be a set of CI statements, and let $c$ be a CI
statement.  Then $\mathcal{C}\vdash c$ if and only if $\mathcal{L}(\mathcal{C}) \supseteq \mathcal{L}(c)$.
\end{theorem}

\vspace{-2mm}
\begin{proof} 
We already know by Proposition~\ref{prop-easy} that
if $\mathcal{C}\vdash c$ then $\mathcal{L}(\mathcal{C}) \supseteq \mathcal{L}(c)$.
We now proceed
to show the other direction. Let us denote $wdec(\mathcal{C})
= \bigcup_{c' \in \mathcal{C}} wdec(c')$ and let $\indep{a}{b}{C} \in
wdec(c)$ with $W=\set{a, b}$. From the assumption $\mathcal{L}(\mathcal{C}) \supseteq \mathcal{L}(c)$ and
Proposition~\ref{lem-closure-equiv}(2) it follows that
$\mathcal{L}(\mathcal{C}) \supseteq \mathcal{L}(a, b| C)$~(1).  By
Proposition~\ref{lem-closure-equiv}(1) it suffices to show that
$\indep{a}{b}{C} \in wdec(\mathcal{C})^+$. However, we will prove the
stronger statement $\forall V \in [C, \overline{W}]: \indep{a}{b}{V}
\in wdec(\mathcal{C})^+$ by downward induction on the lattice $[C,
\overline{W}]$. 

For the base case we need to show that
$\indep{a}{b}{\overline{W}} \in wdec(\mathcal{C})^+$.  By $(1)$
$\overline{W}$ is in $\mathcal{L}(\mathcal{C})$. Hence, by
Proposition~\ref{lem-closure-equiv}(1), there exists a CI statement
$\indep{a}{b}{C'} \in wdec(\mathcal{C})$ such that $\overline{W} \in
\mathcal{L}(a, b | C')$. Now, since $C' \subseteq \overline{W}$, we
can derive $\indep{a}{b}{\overline{W}}$ through \emph{strong union}. 

For the induction step, let $C \subseteq V \subset \overline{W}$. The induction
hypothesis states that for all $V'$ with $V \subset V' \subseteq
\overline{W}$ one has $\indep{a}{b}{V'} \in wdec(\mathcal{C})^+$. By
$(1)$ $V$ is in $\mathcal{L}(\mathcal{C})$. Hence, by
Proposition~\ref{lem-closure-equiv}(1), there exists a CI statement
$\indep{a'}{b'}{C'} \in wdec(\mathcal{C})$ such that $V \in
\mathcal{L}(a', b' | C')$. Since $C' \subseteq V$ we can use
\emph{strong union} to derive $\indep{a'}{b'}{V}$. Let $W' = \set{a',
b'}$. Note that $W' \cap V = \emptyset$. We distinguish three cases:

\begin{itemize}
\item $W' = W$. Then we are done.
\item Exactly one of the two elements in $W'$ is not in $W$. Without
loss of generality let this element be $b'$. Then we can use \emph{contraction} 
on the statements $\indep{a}{b'}{V}$ and
$\indep{a}{b}{Vb'}$ (the latter is in $wdec(\mathcal{C})^+$ by the induction hypothesis)
to derive $\indep{a}{b'b}{V}$, and finally \emph{decomposition} to derive
$\indep{a}{b}{V}$.
\item Both elements in $W'$ are not in $W$. We can use
\emph{strong contraction} on the statements $\indep{a'}{b'}{V}$,
$\indep{a}{b}{Va'}$, and $\indep{a}{b}{Vb'}$ (the latter two are in
$wdec(\mathcal{C})^+$ by the induction hypothesis) to derive $\indep{a}{b}{V}$.
\end{itemize}

This concludes the proof.
\end{proof}

\begin{example}
Let $S = \set{a, b, c, d}$, let $\mathcal{C} = \set{\indep{a}{b}{\emptyset}, \indep{c}{d}{a}, \indep{c}{d}{b}}$ and let $c = \indep{c}{d}{\emptyset}$. We can derive $c$ from $\mathcal{C}$ using the inference rule \emph{strong contraction}. In addition, $\mathcal{L}(\mathcal{C}) = \set{\emptyset, c, d, cd} \cup \set{a, ab} \cup \set{b, ab} = \set{\emptyset, a, b, c, d, ab, cd}$ and $\mathcal{L}(c) = \set{\emptyset, a, b, ab}$, and, therefore, $\mathcal{L}(\mathcal{C}) \supseteq \mathcal{L}(c)$.
\end{example}

\section{The Additive Implication Problem for CI Statements}
\label{additive-implication}

An important result in the study of the implication problem relative to the class of discrete probability measures was gained by Studen{\'y} who linked it to an additive implication problem (Studen{\'y}~\cite{STU:2005}).  More specifically, it was shown that for every CI statement $\indep{A}{B}{C}$, a discrete probability measure $P$ satisfies $\indep{A}{B}{C}$ if and only if the \emph{multi-information function}\footnote{The multi-information function of a probability measure will be formally defined in Section~6.} $M_P$ induced by $P$ satisfies the equality $M_P(C)+M_P(ABC)=M_P(AC)+M_P(BC)$.  Thus, the \emph{multiplication-based} probabilistic CI implication problem was related to an \emph{addition-based} implication problem.  It is this duality that is at the basis of the results developed in this section.  However, rather than immediately focusing on specific classes of \emph{multi-information functions}, which is
what we pursue in Section~\ref{prob-con-ind-imp}, we first consider the additive implication problem for CI statements relative to arbitrary classes of real-valued functions.

By a \emph{real-valued function}, we will always mean a function
$F:2^S\rightarrow \mathbf{R}$, i.e., a function that maps each subset
of $S$ into a real number.

\begin{definition}
\label{def-differentialconstraints}
Let $\indep{A}{B}{C}$ be a CI statement, and let $F$ be a real-valued
function. We say that $F$ \emph{a-satisfies} $\indep{A}{B}{C}$, and
write $\amodels_{F} \indep{A}{B}{C}$, if $F(C) + F(ABC) = F(AC) +
F(BC)$.
\end{definition}

Relative to the notion of \emph{a-satisfaction}, we can now define the
\emph{additive implication problem} for conditional independence statements.

\begin{definition}[Additive implication problem]
\label{def-logicalimplication}
Let $\mathcal{C}$ be a set of CI statements, let $c$ be a CI
statement, and let $\mathcal{F}$ be a class of real-valued functions.
We say that $\mathcal{C}$ \emph{a-implies} $c$ relative to
$\mathcal{F}$, and write $\mathcal{C} \amodels_{\mathcal{F}} c$, if
each function $F \in \mathcal{F}$ that \emph{a-satisfies} the CI statements
in $\mathcal{C}$ also \emph{a-satisfies} the CI statement $c$.
\end{definition}

We now define the notion of density of a real-valued function.  The
density is again a real-valued function and plays a crucial role in
reasoning about additive implication problems.

\begin{definition}
\label{def-densities}
Let $F$ be a real-valued function.  The \emph{density}\footnote{What we call the \emph{density} is sometimes referred to as the \emph{M\"{o}bius inversion} of a real-valued function.} of $F$ is the real-valued function $\Delta F$ defined by $\DeltaF{F}{X} = \sum_{X\subseteq U\subseteq S} (-1)^{|U|-|X|}F(U)$, for each $X \subseteq S$.
\end{definition}

The following relationship between a real-valued function and its
\emph{density} justifies the name. 

\begin{proposition}
\label{prop-densities}
Let $F$ be a real-valued function.  Then, for each $X\subseteq S$,
$F(X)=\sum_{X\subseteq U\subseteq S}\, \DeltaF{F}{U}$.
\end{proposition}

The \emph{a-satisfaction} of a real-valued function for a CI statement can
be characterized in terms of an equation involving its density
function.  This characterization is central in developing our results
and is a special case of a more general result by Sayrafi and Van
Gucht who used it in their study of the \emph{frequent itemset 
mining problem} (Sayrafi and Van Gucht \cite{SV:2005c}).

\begin{proposition}
\label{prop-differentialsasdensities}
Let $\indep{A}{B}{C}$ be a CI statement and and let $F$ be a
real-valued function.  Then, $\amodels_{F} \indep{A}{B}{C}$ if and only
if $\sum_{U \in \mathcal{L}(A, B | C)} \DeltaF{F}{U} = 0.$
\end{proposition}

\section{Properties of Classes of Functions - Soundness and Completeness}
\label{sec-properties}

In this section we study properties of classes of real-valued
functions that guarantee soundness and completeness of 
$\mathcal{A}$, respectively, for the additive implication problem.  How these results relate to 
probabilistic conditional independence implication will become clear in
Section~\ref{soundcompleteA} and Section~\ref{ci-completeness}.

\subsection{Soundness}
\label{subsec-soundness}

First, we define the notion of soundness of system $\mathcal{A}$ for a
given class of real-valued functions.

\begin{definition}[Soundness]
\label{def-soundness}
Let $\mathcal{F}$ be a class of real-valued functions. We say that
 $\mathcal{A}$ is \emph{sound} relative to $\mathcal{F}$ if, for each
 set $\mathcal{C}$ of CI statements and each CI statement $c$, we have
 that $\mathcal{C}\vdash c$ implies
 $\mathcal{C} \amodels_{\mathcal{F}}c$.
\end{definition}

In order to characterize soundness we introduce the following property
of classes of real-valued functions.

\begin{definition}[Zero-density property]
\label{def-strongunion-zero-density-property}
Let $\mathcal{F}$ be a class of real-valued functions.  We say that
$\mathcal{F}$ has the \emph{zero-density property} if, for each $F\in
\mathcal{F}$, for each CI statement $c$, and for each $U \in
\mathcal{L}(c)$, one has that if $\amodels_F c$, then $\DeltaF{F}{U}=0$.
\end{definition}

We can now provide various characterizations of the
soundness of inference system $\mathcal{A}$ for the additive implication problem
for CI statements.

\begin{theorem}
\label{theo-augment-soundness}
Let $\mathcal{F}$ be a class of real-valued functions. Then, the following statements are equivalent:

\begin{enumerate}
\item[(1)] Strong union and decomposition are sound inference rules relative to $\mathcal{F}$ for the additive implication problem;
\item[(2)] $\mathcal{F}$ has the zero-density property; and
\item[(3)] $\mathcal{A}$ is sound relative to $\mathcal{F}$ for the additive implication problem.
\end{enumerate}
\end{theorem}

\begin{proof}
We first prove that statement (1) implies statement (2).  Let
$F\in \mathcal{F}$, let $\indep{A}{B}{C}$ be a CI statement, and
assume $\amodels_{F} \indep{A}{B}{C}$. We now show that
$\DeltaF{F}{V}=0$ for each $V \in \mathcal{L}(A, B | C)$. The proof
goes by downward induction on the semi-lattice $\mathcal{L}(A, B |
C)$. First, we observe that by Lemma~\ref{lem-witnessdec},
$\mathcal{L}(A, B | C) = \bigcup_{W \in \mathcal{W}(A, B | C)} [C,
\overline{W}]$.  Hence, for the base case we must prove that
$\DeltaF{F}{\overline{W}}=0$ for each $W \in \mathcal{W}(A, B|
C)$. Let $W = \set{a, b} \in \mathcal{W}(A, B | C)$. $\indep{a}{b}{C}$
is derivable from $\indep{A}{B}{C}$ using the inference rule
\emph{decomposition} and therefore $\amodels_F \indep{a}{b}{C}$. Since
\emph{strong union} is assumed sound and $C \subseteq \overline{W}$ it
follows that $\amodels_F\indep{a}{b}{\overline{W}}$. Since
$\mathcal{L}(a, b | \overline{W}) = \set{\overline{W}}$ we can invoke 
Proposition~\ref{prop-differentialsasdensities} to conclude that
$\DeltaF{F}{\overline{W}}=0$.  
For the induction step, let $V \in
\mathcal{L}(A, B | C)$. The induction hypothesis states that
$\DeltaF{F}{U}=0$ for all $U\in \mathcal{L}(A, B | C)$ that are strict
supersets of $V$. Similar to the base case, we can infer that
$\amodels_F \indep{A'}{B'}{V}$ with $A', B',$ and $V$ pairwise
disjoint, $A' \subseteq A$, $B' \subseteq B$, and $C \subseteq
V$. Hence, by Proposition~\ref{prop-differentialsasdensities},
$\sum_{U \in \Lat{A'}{B'}{V}} \DeltaF{F}{U}= \DeltaF{F}{V}=0$. Since
for all $U \in \Lat{A'}{B'}{V}$ with $U \neq V$, we have by
Proposition~\ref{prop-easy} that $V \subset U \in \mathcal{L}(A, B |
C)$ and, thus, $\DeltaF{F}{U}=0$ by the induction hypothesis.

We now prove that statement (2) implies statement (3). Let
$\mathcal{C}$ be a set of CI statements, let $c$ be a CI statement,
and assume that $\mathcal{C} \vdash c$.  Since $\mathcal{F}$ has the
zero-density property, we have that for each $F \in \mathcal{F}$, if $\amodels_F
\mathcal{C}$ then for each $U \in \mathcal{L}(\mathcal{C})$,
$\DeltaF{F}{U}=0$. From $\mathcal{C} \vdash c$ and
Proposition~\ref{prop-easy}, we have $\mathcal{L}(C) \supseteq
\mathcal{L}(c)$. Hence, for all $F \in \mathcal{F}$ we have 
that if $F$ \emph{a-satisfies} every CI statements in $\mathcal{C}$, then $F$ \emph{a-satisfies} $c$.
Thus, $\mathcal{C} \amodels_{\mathcal{F}}c$.

Finally, statement (1) follows trivially from (3).
\end{proof}

\subsection{Completeness}
\label{subsec-completeness}

As with soundness in Subsection~\ref{subsec-soundness}, we begin with
the definition of the notion of \emph{completeness} of inference system $\mathcal{A}$
for a given class of real-valued functions.

\begin{definition}[Completeness]
\label{def-completeness}
Let $\mathcal{F}$ be a class of real-valued functions. We say that
$\mathcal{A}$ is \emph{complete} for the additive implication problem
for CI statements relative to $\mathcal{F}$ if, for each set
$\mathcal{C}$ of CI statements and each CI statement~$c$, one has that
$\mathcal{C} \amodels_{\mathcal{F}}c$ implies $\mathcal{C} \vdash c$.
\end{definition}

We now introduce certain special real-valued functions that are at the
basis of defining a property guaranteeing completeness of system
$\mathcal{A}$.

\begin{definition}
\label{def-kronecker}
Let $V\subseteq S$.  The \emph{Kronecker-density function} of $V$,
denoted $\delta_V$, is the real-valued function such that
$\delta_V(V)=1$ and $\delta_V(X)=0$ if $X\neq V$.  The
\emph{Kronecker-induced function} of $V$, denoted $F_{V}$, is the
real-valued function whose density function is the Kronecker density
function of $V$, i.e., for each $X\subseteq S$, $F_{V}(X) =
\sum_{X\subseteq U\subseteq S}\, \delta_V(U)$, for each $X\subseteq S$.
\end{definition}

We can now define a property on classes of real-valued functions that
we will show to guarantee the completeness of system $\mathcal{A}$ for
the additive implication problem.

\begin{definition}[Kronecker property]
\label{def-kroneckerproperty}
Let $\mathcal{F}$ be a class of real-valued functions, and let $\Omega
\subseteq 2^S$.  We say that $\mathcal{F}$ has the \emph{Kronecker
property} on $\Omega$ if, for each $U \in \Omega$, there exists a $c_U
\in \mathbf{R}$ ($c_U\neq 0$), and a set $D_U =\set{d_V \in \mathbf{R}
\mid V \notin \Omega}$ such that the following real-valued function is
in $\mathcal{F}$:
\begin{eqnarray*}
F_{\Omega,c_U,D_U} := c_UF_{U} +  \sum_{\substack{V \subseteq S \\ V \notin \Omega}}d_VF_{V}.
\end{eqnarray*}
Note that for all $X \in \Omega$, $\Delta F_{\Omega,c_U,D_U}(X) = c_U$
if $X = U$ and $\Delta F_{\Omega,c_U,D_U}(X) = 0$ if $X \neq U$.
\end{definition}

\looseness=-1 Let $\Omega^{(2)}$ be the set of all subsets of $S$ that lack at
least two of their elements, i.e., $\Omega^{(2)}=\set{V\subset S\mid
|V|\leq |S|-2}$.  We can now prove that the Kronecker property on $\Omega^{(2)}$ implies the completeness of system $\mathcal{A}$.

\begin{theorem}
\label{theo-disjoint-kroneckercomplete}
\looseness=-1 Let $\mathcal{F}$ be a class of real-valued functions.
If $\mathcal{F}$ has the Kronecker property on $\Omega^{(2)}$, then
system $\mathcal{A}$ is complete for the additive implication problem
for CI statements relative to $\mathcal{F}$.
\end{theorem}

\begin{proof}
Assume that $\mathcal{F}$ has the Kronecker property on $\Omega^{(2)}$
but that $\mathcal{A}$ is not complete.  Then there exists a set
$\mathcal{C}$ of CI statements and a CI statement $c$ such that
$\mathcal{C} \amodels_{\mathcal{F}} c$ but $\mathcal{C} \not \vdash c$,
or, equivalently by Theorem~\ref{theo-derivablelatdec},
$\mathcal{L}(c) \nsubseteq \mathcal{L}(\mathcal{C})$. Let $U \in \mathcal{L}(c) - \mathcal{L}(\mathcal{C})$.  $U$ must be an element in
$\Omega^{(2)}$ by Lemma~\ref{lem-witnessdec}. Since $\mathcal{F}$ has
the Kronecker property on $\Omega^{(2)}$, we know that there exists a
$c_U \in \mathbf{R}$ ($c_U\neq 0$), and a set $D_U =\set{d_V \in
\mathbf{R} \mid V \notin \Omega^{(2)}}$ such that
$F_{\Omega^{(2)},c_U,D_U} \in \mathcal{F}$. By Definition~\ref{def-kroneckerproperty}, $\Delta F_{\Omega^{(2)},c_U,D_U}(X) = 0$ for all other $X \in \Omega^{(2)}$.
From Proposition~\ref{prop-differentialsasdensities} it follows that
$\amodels_{F_{\Omega^{(2)},c_U,D_U}}\, \mathcal{C}$, but
$\not\amodels_{F_{\Omega^{(2)},c_U,D_U}}\, c$, a contradiction to
$\mathcal{C}\amodels_{\mathcal{F}} c$.
\end{proof}

The following example demonstrates the zero-density and Kronecker properties.

\begin{example}
Let $S = \set{a, b, c}$, let $\mathcal{F}_1 = \set{F_{\emptyset}, F_{a}, F_{b}, F_{c}}$ and $\mathcal{F}_2 = \set{F_x}$, where the densities for each real-valued function are given by the table in Figure~\ref{table-1}. The densities of the remaining subsets of $S$ are assumed to be $0$ for each function. Now, $\Omega^{(2)} = \set{\emptyset, a, b, c}$ and, therefore, $\mathcal{F}_1$ has the Kronecker property on $\Omega^{(2)}$ since $F_{\Omega^{(2)},c_U,D_U} = F_U$ for all $U \in \Omega^{(2)}$, and  the zero-density property. $\mathcal{F}_2$ does not have the Kronecker property.  It also does not have the zero-density property as $\amodels_{F_x} \indep{b}{c}{\emptyset}$ but $\DeltaF{F_x}{\emptyset} \neq 0$.
\vspace{-3mm}
\begin{figure}[h!]
\begin{center}
\begin{tabular}{|c|cccccc|}
\hline  
&  \ \T\B \ & $\emptyset$ \ & \ $\set{a}$ \ & \ $\set{b}$ \ & \ $\set{c}$ \ & \\ 
\hline  
\T \ $\Delta F_{\emptyset}$ \ \ & &  0.1 & 0 & 0 & 0 &  \\ 
\hline
\T $\Delta F_{a}$   & & 0 & -0.3 & 0 & 0  & \\ 
\hline 
\T $\Delta F_{b}$  & & 0 & 0 & -0.6 & 0 &  \\ 
\hline  
\T $\Delta F_{c}$  & & 0 & 0 & 0 & 0.9 &  \\ 
\hline  
\T $\Delta F_x$  & & -0.2 & 0.2 & 0.6 & 0.3 &  \\ 
\hline 
\end{tabular} 
\caption{\label{table-1} Densities of several real-valued functions.}
\end{center}
\end{figure}
\vspace{-4mm}
\end{example}

\section{The Conditional Independence Implication Problem}
\label{prob-con-ind-imp}

While the theory presented so far has been concerned with the additive
implication problem for CI statements, it is also applicable to the
conditional independence implication problem.  The link between these
two problems is made with the concept of \emph{multi-information functions}
(Studen{\'y} \cite{STU:2005}) induced by probability measures.
In this paper we will restrict ourselves to the class of \emph{discrete}
probability measures.

\begin{definition}
A \emph{probability model} over $S= \set{s_1,\ldots, s_n}$ is a pair
$(dom, P)$, where $dom$ is a domain mapping that maps each $s_i$ to a
finite domain $dom(s_i)$, and $P$ is a probability measure having
$dom(s_1)\times \cdots \times dom(s_n)$ as its sample space.  For
$A=\{a_1,\ldots,a_k\}\subseteq S$, we will say that $\mathbf{a}$ is a
domain vector of $A$ if $\mathbf{a}\in dom(a_1)\times \cdots \times
dom(a_k)$.
\end{definition}

\looseness=-1 In what follows, we will only refer to probability measures, keeping their  probability models implicit.

\begin{definition}
\label{def-prob-semantics}
Let $\indep{A}{B}{C}$ be a CI statement, and let $P$ be a probability
measure.  We say that $P$ \emph{m-satisfies} $\indep{A}{B}{C}$, and
write $\mmodels_{P}\indep{A}{B}{C}$, if for every domain vector
$\mathbf{a}$, $\mathbf{b}$, and $\mathbf{c}$ of $A$, $B$, and $C$,
respectively, $P(\mathbf{c})P(\mathbf{a}, \mathbf{b}, \mathbf{c}) =
P(\mathbf{a}, \mathbf{c})P(\mathbf{b}, \mathbf{c})$.
\end{definition}

Relative to the notion of \emph{m-satisfaction} we can now define the
\emph{probabilistic conditional independence implication problem}.

\begin{definition}[Probabilistic conditional independence implication problem]
Let $\mathcal{C}$ be a set of CI statements, let $c$ be a CI
statement, and let $\mathcal{P}$ be the class of discrete probability measures.
We say that $\mathcal{C}$ \emph{m-implies} $c$ relative to
$\mathcal{P}$, and write $\mathcal{C} \mmodels_{\mathcal{P}} c$, if
each function $P \in \mathcal{P}$ that \emph{m-satisfies} the CI statements
in $\mathcal{C}$ also \emph{m-satisfies} the CI statement $c$.  The set
$\set{c \mid \mathcal{C} \mmodels_{\mathcal{P}} c}$ will be denoted by
$\mathcal{C}^*$.
\end{definition}

Next, we define the \emph{multi-information function} induced by a
probability measure~(Studen{\'y}~\cite{STU:2005}), which is based on the Kullback-Leibler divergence~(Kullback and Leibler~\cite{KL:1951}).

\begin{definition}
Let $P$ and $Q$ be two probability measures over a discrete sample space.
Then, the relative entropy (Kullback-Leibler divergence) $H$ is defined as 
$$H(P | Q) := \sum_{\mathbf{x}} \set{ P(\mathbf{x}) \mbox{ log} \frac{P(\mathbf{x})}{Q(\mathbf{x})}, \ P(\mathbf{x}) > 0},$$ with $\mathbf{x}$ ranging over all elements of the discrete sample space.
\end{definition}

\begin{definition}
\label{def-multiinformation}
Let $P$ be a probability measure, and let $H$ be the relative
entropy.  The \emph{multi-information
function} $M_P: 2^S \rightarrow [0, \infty]$ induced by $P$ is
defined as
$$M_P(A) := H(P^A | \prod_{a \in A} P^{\set{a}}),$$ 
for each non-empty subset $A$ of $S$ and $M_P(\emptyset) =
0$.\footnote{Here, $P^A$ and $P^{\set{a}}$ denote the marginal
probability measures of $P$ over $A$ and $\set{a}$,
respectively.} 
\end{definition}

The class of multi-information functions induced by the class of discrete probability measures $\mathcal{P}$ will be denoted by $\mathcal{M}$.
We can now state the fundamental result of Studen{\'y} that couples
the probabilistic CI implication problem with the additive implication problem
for CI statements relative to $\mathcal{M}$.

\begin{theorem}[Studen{\'y} \cite{STU:2005}]
\label{cor-imp-problem-equiv}
Let $\mathcal{C}$ be a set of CI statements and let $c$ be a CI
statement.  Then, $\mathcal{C} \amodels_{\mathcal{M}} c$ if and only if
$\mathcal{C} \mmodels_{\mathcal{P}} c$.
\end{theorem}

\section{Saturated CI Statements - Soundness and Completeness of $\mathcal{A}$}
\label{soundcompleteA}

In this section we show that system $\mathcal{A}$ is sound and
complete for the probabilistic CI implication problem for
\emph{saturated} CI statements. We recall that a CI statement
$\indep{A}{B}{C}$ is saturated if $ABC=S$.  We begin by showing the
following technical lemma.

\begin{lemma}
\label{zero-density-wrt-ci}
The class of multi-information functions $\mathcal{M}$ induced by the class of discrete probability measures has the zero-density property
with respect to saturated CI statements.
\end{lemma}

\begin{proof}
We have to show that for each saturated CI statements $c$, for each $M
\in \mathcal{M}$, and for each $U \in \mathcal{L}(c)$, if $\amodels_{M}
c$, then $\DeltaF{M}{U}=0$.  The semi-graphoid inference rules are
sound relative to the class of probability measures.  Hence, in
particular, by Theorem~\ref{cor-imp-problem-equiv}, \emph{weak union} 
is sound relative to $\mathcal{M}$, i.e., $\set{ \indep{AD}{B}{C} }
\amodels_{\mathcal{M}} \indep{A}{B}{CD}$.  Let $M \in \mathcal{M}$, let
$\Delta M$ be the corresponding density function, and let $\amodels_M
\indep{A}{B}{C}$ with $ABC = S$. In addition, let $\indep{A}{B}{C}$ be
non-trivial since the proposition is obviously true for trivial CI
statements.  We will prove by downward induction on the semi-lattice
$\mathcal{L}(A, B | C)$ that $\DeltaF{M}{U}=0$ for each $U \in
\mathcal{L}(A, B | C)$. Note that this proof is similar to the proof
of Proposition~\ref{theo-augment-soundness}. (Here, \emph{weak union} is used
instead of \emph{decomposition} and \emph{strong union}). 

For the base case, we
show for each $W \in \mathcal{W}(A, B | C)$ that
$\DeltaF{M}{\overline{W}}=0$.  Let $W = \set{a, b}$.  By repeatedly
applying \emph{weak union} we can derive $\amodels_{M}
\indep{a}{b}{\overline{W}}$ because $ABC = S$. Now, since
$\mathcal{L}(a, b | \overline{W}) = \set{\overline{W}}$ we can
conclude that $\DeltaF{M}{\overline{W}}=0$.  

For the induction step, let $V \in \mathcal{L}(A, B | C)$. The induction hypothesis states
that $\DeltaF{M}{U}=0$ for each $U\in \mathcal{L}(A, B | C)$ with $U$
a strict superset of~$V$.  From the given CI statement
$\indep{A}{B}{C}$ we can derive, again by \emph{weak union},
$\indep{A'}{B'}{V}$ with $VA'B' = S$ since $V-C\subseteq
AB$. Since $\mathcal{L}(A', B' | V)$ contains only $V$ and strict
supersets $V'$ of $V$, with $V' \in \mathcal{L}(A, B | C)$, we can
conclude that $\sum_{U\in\Lat{A'}{B'}{V}}
\DeltaF{F}{U}=\DeltaF{F}{V}=0$ by the induction hypothesis.
\end{proof}

We are now in the position to prove that inference system $\mathcal{A}$ is sound and complete for the probabilistic implication problem for \emph{saturated} conditional independence statements.

\begin{theorem}
\label{thm-infsys-sat-complete}
$\mathcal{A}$ is sound and complete for the probabilistic conditional independence
implication problem for saturated CI statements.
\end{theorem}

\begin{proof}
The soundness follows directly from Lemma~\ref{zero-density-wrt-ci},
Theorem~\ref{theo-augment-soundness}, and
Theorem~\ref{cor-imp-problem-equiv}. To show completeness, notice that the semi-graphoid axioms are derivable under inference system $\mathcal{A}$.
Furthermore, Geiger and Pearl proved that the semi-graphoid axioms are
complete for the probabilistic conditional independence implication problem for
\emph{saturated} CI statements (Geiger and Pearl~\cite{GP:1993}). 
\end{proof}

\section{CI Statements - Completeness of $\mathcal{A}$}
\label{ci-completeness}

In this section we will show that inference system $\mathcal{A}$ is complete
for the probabilistic conditional independence implication problem.  We first prove that $\mathcal{M}$ has the Kronecker property on $\Omega^{(2)}$.  To show this, it would be sufficient to construct a
set of discrete probability measures whose induced multi-information
functions are Kronecker-induced functions.  However, instead of taking
this route, we pursue a different approach by first focusing on
results with respect to \emph{saturated} CI statements.
We first need the following simple lemma.

\begin{lemma}
\label{lem-supset-rule-sat}
For $U \subseteq S$, $\set{X \in \Omega^{(2)} \mid X \supseteq U} =
\bigcup_{\substack{U_1 \cup U_2 = \overline{U}\\ U_1 \cap U_2 =
\emptyset}} \mathcal{L}(U_1, U_2 | U)$.
\end{lemma}

\begin{proposition}
\label{prop-multi-information-kronecker}
\looseness=-1 Let $\mathcal{F}$ be a class of real-valued functions.
If $\mathcal{A}$ is sound and complete for the additive
implication problem relative to
$\mathcal{F}$ for \emph{saturated} CI statements, then $\mathcal{F}$ has the Kronecker property on
$\Omega^{(2)}$.
\end{proposition}

\begin{proof}
  If $|S| \leq 1$, then $\Omega^{(2)} = \emptyset$ and the statement
  follows trivially. Hence, assume that $|S| \geq 2$.  Suppose that
  $\mathcal{A}$ is sound and complete for saturated CI statements but that $\mathcal{F}$ does not
  have the Kronecker property on $\Omega^{(2)}$.  Then there exists a
  set $U \in \Omega^{(2)}$ such that for each $c_U \in \mathbf{R}$
  ($c_U \neq 0$), and for each set $D_U=\set{d_V \in \mathbf{R} \mid V
  \notin \Omega^{(2)}}$ we have $F_{\Omega^{(2)},c_U,D_U}\not\in
  \mathcal{F}$.  Now, let $\mathcal{C}$ be the set of saturated CI statements
  $$\set{\ \indep{U}{\overline{U}}{\emptyset}\ } \ \cup
  \bigcup_{\substack{U_1 \cup U_2 = U \\ U_1 \cap U_2 = \emptyset}}
  \set{\ \indep{U_1}{U_2}{\overline{U}}\ } \ \cup $$
$$\bigcup_{v \in \overline{U}}\ \bigcup_{\substack{V_1 \cup V_2 =
    \overline{U}-\set{v} \\ V_1 \cap V_2 = \emptyset}} \set{\ 
  \indep{V_1}{V_2}{U \cup \set{v}}\ },$$
and let $c$ be the saturated CI statement $\indep{U_1}{U_2}{U}$ for
some non-empty sets $U_1$ and $U_2$.  Notice that such sets exist
because $|\overline{U}| \geq 2$.  By Lemma~\ref{lem-supset-rule-sat}
it is $\mathcal{L}(\mathcal{C}) = \Omega^{(2)} - \set{U}$ and $U \in
\mathcal{L}(c) \subseteq \Omega^{(2)}$ and therefore $\mathcal{L}(c)
\nsubseteq \mathcal{L}(\mathcal{C})$. Hence, by
Theorem~\ref{theo-derivablelatdec}, $\mathcal{C} \nvdash c$.  We now
show that $\mathcal{C}\amodels_{\mathcal{F}} c$ to obtain the 
contradiction to the completeness of $\mathcal{A}$. If there does not exist an $F \in \mathcal{F}$ which
\emph{a-satisfies} $\mathcal{C}$ we are done because then $\mathcal{C}
\amodels_{\mathcal{F}} c$ follows trivially.  Thus, let $F$ be in
$\mathcal{F}$ and assume that $\amodels_F \mathcal{C}$.  Since
$\mathcal{A}$ is sound relative to $\mathcal{F}$ for 
saturated CI statements, we know by
Theorem~\ref{theo-augment-soundness} that $\mathcal{F}$ has the
zero-density property.  Thus, $\DeltaF{F}{X}=0$ for each $X \in \Omega^{(2)}$
with $X \neq U$. But then $\DeltaF{F}{U}=0$ since otherwise there
would exist a $c_U \in \mathbf{R}$, $c_U = \DeltaF{F}{U} \neq 0$, and
a set $D_U=\set{d_V \in \mathbf{R} \mid V \notin \Omega^{(2)}}$ such
that $F_{\Omega^{(2)},c_U,D_U} = F \in \mathcal{F}$.  Hence, $F$ must
be a function whose density is zero on every element of
$\Omega^{(2)}$.  Thus, $\amodels_F c$ and it follows that
$\mathcal{C} \amodels_{\mathcal{F}} c$.
\end{proof}

\looseness=-1 The completeness of $\mathcal{A}$ for the CI implication problem can now be proved based on the previous results.

\begin{theorem}
\label{thm-complete-main}
$\mathcal{A}$ is complete for the probabilistic conditional independence
implication problem.
\end{theorem}

\begin{proof}
We know from Theorem~\ref{thm-infsys-sat-complete} that $\mathcal{A}$ is sound and complete relative to $\mathcal{M}$ for saturated CI statements.  Now, by
Proposition~\ref{prop-multi-information-kronecker}, $\mathcal{M}$ has
the Kronecker property on $\Omega^{(2)}$.  Finally, through
Theorem~\ref{theo-disjoint-kroneckercomplete} and
Theorem~\ref{cor-imp-problem-equiv}, the statement follows.
\end{proof}

\begin{example}
\label{ex-formal-prop}
(Studen{\'y} \cite{STU:2005}) described the following sound inference rule relative to  discrete probability measures which refuted the conjecture (Pearl~\cite{GEI:1988}) that the semi-graphoid axioms are complete for the probabilistic CI implication problem:

$\indep{A}{B}{CD} \wedge \indep{C}{D}{A} \wedge \indep{C}{D}{B} \wedge \indep{A}{B}{\emptyset} \rightarrow
\indep{C}{D}{AB} \wedge \indep{A}{B}{C} \wedge \indep{A}{B}{D} \wedge \indep{C}{D}{\emptyset}.$

By applying \emph{strong contraction} to the statements $\indep{A}{B}{\emptyset}, \indep{C}{D}{A},$ and $\indep{C}{D}{B}$ we can derive the statement $\indep{C}{D}{\emptyset}.$ All the other statements can be derived using \emph{strong union}.
\end{example}

\begin{remark}
\looseness=-1 The inference system $\mathcal{A}$ without \emph{strong contraction} is \emph{not} complete. The consequence $\indep{C}{D}{\emptyset}$ of the clause from Example~\ref{ex-formal-prop} cannot be derived from the antecedents without \emph{strong contraction}.
\end{remark}

\section{Complete Axiomatization of Stable Independence}
\label{stableindependence}

\looseness=-1 When new information is available to a probabilistic system the set of
associated relevant CI statements changes dynamically. However, some
of the CI statements will continue to hold. These CI
statements were termed \emph{stable} by de Waal and van der Gaag
\cite{WAAL:2004}. A first investigation of their structural properties was undertaken by Mat\'{u}\v{s}  who used the term \emph{ascending} conditional independence (Mat\'{u}\v{s} \cite{MATUS:1992}). Every set of CI statements can be partitioned into its \emph{stable} and \emph{unstable} part. We will show that inference system $\mathcal{A}$ is sound and complete for the probabilistic CI implication problem for \emph{stable} conditional independence statements. 

\begin{definition}
Let $\mathcal{C}$ be a set of CI statements, and let
$\mathcal{C}^{SG+}$ be the semi-graphoid closure of
$\mathcal{C}$. Then $\indep{A}{B}{C}$ is said to be \emph{stable} in
$\mathcal{C}$, if $\indep{A}{B}{C'} \in \mathcal{C}^{SG+}$ for all
sets $C'$ with $C \subseteq C' \subseteq S$.
\end{definition}

\begin{theorem}
\label{prop-stable-sound-and-complete}
Let $\mathcal{C}_S$ be a set of stable CI statements. Then, $\mathcal{A}$ is sound and complete for the probabilistic conditional independence
implication problem for $\mathcal{C}_S$, or,
equivalently, $\mathcal{C}_S^* = \mathcal{C}_S^{+}$.
\end{theorem}

\begin{proof}
The soundness follows from Theorem~\ref{theo-augment-soundness} and
from \emph{strong union} and \emph{decomposition} being sound
inference rules relative to $\mathcal{M}$ for stable CI statements.
The completeness follows from Theorem~\ref{thm-complete-main}.
\end{proof}

\begin{remark}
\looseness=-1 The previous result is also interesting with respect to the problem of finding a minimal, non-redundant representation of stable independence relations. Here, lattice-inclusion could aid the lossless compaction of representations of stable CI statements:  $\mathcal{L}(\mathcal{C}_S -\set{c}) = \mathcal{L}(\mathcal{C}_S)$ if and only if $c$ is redundant in $\mathcal{C}_S$.
\end{remark}

\section{Falsification Algorithm}
\label{heuristics}
Theorem~\ref{theo-derivablelatdec} and Theorem~\ref{thm-complete-main} lend
themselves to a \emph{falsification algorithm}, that is, an algorithm which can falsify 
instances of the probabilistic conditional independence implication problem. We consider the following corollary which directly follows from these two results.

\begin{corollary}
\label{cor-lattice-complete}
Let $\mathcal{C}$ be a set of CI statements, and let $\mathcal{P}$ be
the class of discrete probability measures.  If $\mathcal{L}(\mathcal{C})
\nsupseteq \mathcal{L}(c)$, then $\mathcal{C} \nmmodels_{\mathcal{P}} c$.
\end{corollary}

\looseness=-1 If the falsified implications were, on average, only a small
fraction of all those that are falsifiable, the result would be disappointing from a practical point of view. Fortunately, we will not
only be able to show that a large number of implications can be
falsified by the ``lattice-exclusion" criterion identified in
Corollary~\ref{cor-lattice-complete}, but also that polynomial time
heuristics exist that provide good approximations of said
criterion.  

\noindent
{\bf Falsification Criterion}.
{\tt Input}: A set of CI statements $\mathcal{C}$ and a CI statement
$c$. {\tt Test}: if $\mathcal{L}(\mathcal{C}) \nsupseteq
\mathcal{L}(c)$, return ``false", else return ``unknown."

{\bf Heuristic 1}.  {\tt Input}: A set of CI statements $\mathcal{C}$
and a CI statement $\indep{A}{B}{C}$. {\tt Test}: if for each
$\indep{A'}{B'}{C'} \in \mathcal{C}$ it is $C \nsupseteq C'$, return
``false", else return ``unknown."

\noindent
{\bf Heuristic 2}.  {\tt Input}: A set of CI statements $\mathcal{C}$,
and a CI statement $\indep{A}{B}{C}$. {\tt Test}: if there exists one
$W \in \mathcal{W}(A, B | C)$ such that for all $\indep{A'}{B'}{C'}
\in \mathcal{C}$ it is $W \notin \mathcal{W}(A', B' | C')$, return
``false", else return ``unknown."

\looseness=-1 It follows from Lemma~\ref{lem-witnessdec} that if one of the
two heuristics returns ``false,'' then $ \mathcal{L}(\mathcal{C})
\nsupseteq \mathcal{L}(c)$, and therefore $\mathcal{C}
\nmmodels_{\mathcal{P}} c$ by Corollary~\ref{cor-lattice-complete}.

\begin{example}
Let $S$ be a finite set, and $A, B, C$, and $D$ be pairwise disjoint subsets of $S$.
The inference rule \emph{intersection}, 
$\indep{A}{B}{DC} \wedge \indep{A}{D}{BC} \rightarrow \indep{A}{BD}{C}$,
\looseness=-1 is \emph{not} sound relative to the class of discrete probability measures.  Heuristic~1 can reject this instance of the implication problem in polynomial time in the size of $S$.
\end{example}

\begin{remark}
The falsification criterion leads in fact to a \emph{family} of polynomial time heuristics. While Heuristic 1 checks if the unique \emph{meet} (greatest lower bound) of the semi-lattice $\mathcal{L}(c)$ is not in $\mathcal{L}(\mathcal{C})$ and Heuristic 2 if the (potentially multiple) \emph{joins} (least upper bounds) of the semi-lattice $\mathcal{L}(c)$ are not in $\mathcal{L}(\mathcal{C})$, we may select additional elements in the semi-lattice $\mathcal{L}(c)$ that are located between these two extrema to derive more falsification  heuristics.
\end{remark}

\begin{figure}[t!]
  \includegraphics*[viewport=42 110 422 352, scale=0.62]{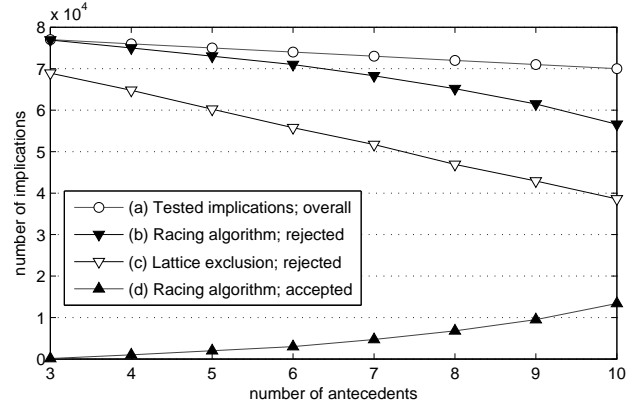}
\vspace{-6mm}
\caption{\label{graph1}Rejection and acceptance curves of the racing and falsification algorithms, respectively, for five attributes.}
\end{figure}
\begin{figure}[t!]
\includegraphics*[viewport=42 110 422 352, scale=0.62]{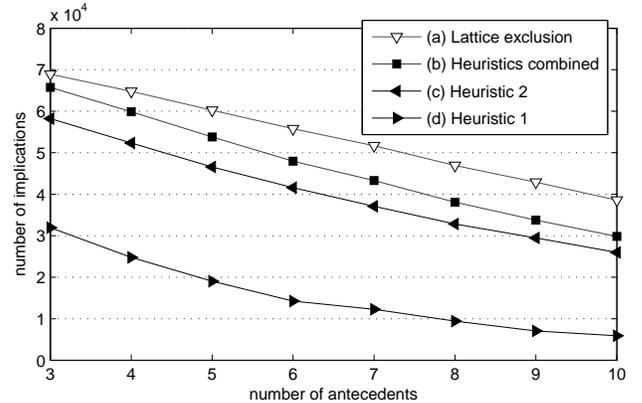}
\vspace{-6mm}
\caption{\label{graph2}Falsifications based on the lattice-exclusion criterion and the heuristics, for five attributes. The combination of the heuristics reaches 95\% of the falsifications of the full-blown lattice exclusion criterion for 3 antecedents down to 77\% for 10 antecedents.}
\end{figure}

With our experiments we want to show that (1) the lattice-exclusion criterion can 
falsify a large fraction of all falsifiable implications, and (2) that the two provided heuristics are good approximation of the full-blown lattice-exclusion criterion.
To make our outcomes comparable to existing results, we
adopted the experimental setup for the \emph{racing algorithm}
from Bouckaert and Studen{\'y}~\cite{BOU:2007} (also using 5 attributes). A thousand sets of antecedents each were generated
by randomly selecting 3 up to 10 elementary CI statements, resulting
in a total of 8000 sets of antecedents.\footnote{An elementary CI
statement is of the form $\indep{a}{b}{C}$, where $a,b\in S$ and $C \subseteq S-\set{a,b}$.} The \emph{falsification algorithm} and the heuristics were run on these sets with each of the remaining elementary CI statements as consequence, one at a time.  Since
there are 80 elementary CI statements for 5 attributes, this resulted
in 77000 implication problems for sets with 3 antecedents, 76000 for
sets with 4 antecedents, down to 70000 for sets with 10 antecedents.

The rejection procedure of the \emph{racing algorithm} is rooted in the
theory of imsets: an instance is rejected
if one of the supermodular functions constructed by the algorithm is a counter-model for this instance. It has exponential running time and might
reject implications that actually \emph{do} hold.  This is a
consequence of the fact that $\mathcal{M}$
is a \emph{strict} subset of the class of all supermodular functions. (See
Examples~4.1~and~6.2 in Studen{\'y}'s monograph~\cite{STU:2005}.)  The
\emph{falsification algorithm} based on Corollary~\ref{cor-lattice-complete}, on the other hand, ensures that if an instance of the implication problem is rejected, then it is guaranteed not to be valid. 

Figure~\ref{graph1} shows the rejection curves of the \emph{racing algorithm} (b) and the  \emph{falsification algorithm} (c), respectively, and the acceptance curve of the racing algorithm (d).  The area between the two rejection curves can be interpreted as the ``decision gap", i.e., the amount of instances of the implication problem for which the validity is unknown. The curve marked with circles (a) 
depicts the total number of tested instances.  Figure~\ref{graph2}
depicts the rejection curves for the \emph{falsification algorithm} (a), 
for the combination of Heuristic 1 and Heuristics 2 (b), and for
Heuristic 2 (c) and Heuristic 1 (d) run separately.  The
combination of the heuristics compares favorable with the
full-blown \emph{falsification criterion}. The experiments also show that Heuristic 2 is more effective than
Heuristic 1.

\section{Conclusion and Future Work}
\label{discussion}

\looseness=-1 A complete inference system for the probabilistic 
conditional independence implication problem was presented and related
to the lattice-exclusion criterion. We derived polynomial time
approximations that can be used as a preprocessing step to efficiently
shrink the search space of possibly valid inferences.  We already
have experimental evidence that our approach scales to much
larger instances of the implication problem than
those reported on in this paper. This could, for instance, provide insights into combinatorial bounds for the number of (stable) CI structures. 
The falsification algorithm and the heuristics can be combined with algorithms that infer valid implications, like the one based on  structural imsets which is used as part of the \emph{racing algorithm} \cite{BOU:2007}.  In addition, the lattice exclusion criterion and the heuristics can be utilized to store information about conditional independencies more efficiently, using non-redundant representations. Overall, we believe that the lattice-theoretic framework for reasoning about conditional independence is a novel and powerful tool.  We conjecture that there are interesting connections between our theory and Studen{\'y}'s theory of imsets which we will continue to investigate.

\subsubsection*{Acknowledgments}

We thank Remco Bouckaert for providing us with the source code of the \emph{racing algorithm}.  We also want to thank Milan
Studen{\'y} for the information he provided us concerning the complete
axiomatization of the implication problem for
four attributes, Franti\v{s}ek Mat\'{u}\v{s} for helpful feedback on an earlier draft, and  the anonymous reviewers whose comments helped to improve the quality of the paper.

\bibliographystyle{abbrv}

\end{document}